\documentclass{article}
\usepackage{arxiv}

\usepackage[numbers]{natbib}
\usepackage{graphicx}
\usepackage{subcaption}
\usepackage{algorithmic}
\usepackage{algorithm}
\usepackage{threeparttable}
\usepackage{float}
\usepackage{tabularx}
\usepackage{siunitx}
\usepackage{multirow}
\usepackage{booktabs}
\usepackage{amsmath, amsthm, amssymb}
\newtheorem{theorem}{Theorem}

\newtheorem{assumption}{Assumption}

\newcolumntype{Y}{>{\centering\arraybackslash}X}

\begin{document}

\title{DiffSlack: Learning under Nonlinear Inequality Constraints via Learnable Slack Variables}

\author{Ziqian Wang
\And 
Chenxi Fang
\And 
Zhen Zhang
\thanks{Corresponding author: Z. Zhang (e-mail: zzhang@tsinghua.edu.cn). Ziqian Wang and Zhen Zhang are with the State Key Laboratory of Tribology in Advanced Equipment, Tsinghua University, Beijing 100084, China and are with the Beijing Key Laboratory of Transformative High-end Manufacturing Equipment and Technology, Department of Mechanical Engineering, Tsinghua University, Beijing 100084, China. Chenxi Fang is with the Automotive Electronics Business Unit at Hirain Inc., China.}
}
\maketitle

\begin{abstract}
Enforcing nonlinear inequality constraints in neural networks remains challenging, especially when the output is subject to many coupled constraints. Existing hard constraint methods often impose structural restrictions on the constraint set or introduce substantial computational overhead for large-scale nonlinear problems. Here, we propose DiffSlack, a differentiable projection layer for nonlinear inequality-constrained neural prediction. DiffSlack reformulates inequalities as equalities with learnable slack variables, which are predicted as part of the augmented network output and provide a data-driven warm start for damped Gauss-Newton projection. The projection layer maps raw predictions onto the augmented feasible manifold while preserving end-to-end differentiability. A two-stage curriculum further stabilizes training and improves constraint satisfaction. We evaluate DiffSlack on vehicle path planning with 200 nonlinear inequality constraints from collision avoidance, curvature limits, and waypoint spacing. Compared with existing learning-based baselines, DiffSlack achieves a higher planning success rate and stronger geometric constraint satisfaction under a comparable inference budget. Ablation studies further show that the hard projection layer reduces sensitivity to supervision quality. Closed-loop tracking in CARLA and real-world vehicle experiments confirms the executability of the generated trajectories. These results demonstrate that DiffSlack provides a practical and scalable approach to embedding hard inequality constraints into neural networks for engineering applications.
\end{abstract}

\section{Introduction}

Embedding domain knowledge into neural networks has become an important direction for improving the reliability and physical consistency of learned models \citep{GAO2026115749, cui2025knowledge, zhang2024training}. Feasibility constraints are a central form of such knowledge. They encode physical laws, safety limits, and operational requirements that model outputs must satisfy during deployment. However, standard neural networks are primarily trained to minimize prediction losses and do not guarantee constraint satisfaction at inference time. As a result, even accurate predictors may produce infeasible or unsafe outputs when deployed in safety-critical systems. This gap between predictive accuracy and feasibility has motivated growing interest in constrained neural prediction \citep{projectnet2023, van2023constraint, kim2022projectionaware}.

Existing approaches can be broadly grouped into three categories, as illustrated in Fig.~\ref{fig:method}. The first incorporates constraints as soft penalty terms in the training objective \citep{park2023selfsupervised, fioretto2020lagrangian, raissi2019physicsinformed}. This strategy is simple and widely applicable, but constraint satisfaction depends on penalty weights and training data coverage. Consequently, violations can still occur at inference time, especially under distribution shift or competing task objectives. The second category applies post-hoc safety filters, such as control barrier functions and predictive safety shields \citep{zhang2025control, silvestri2024unify, wabersich2023predictive, ames2017control}. These methods can enforce safety during deployment, but the correction is usually applied outside the prediction model. This weakens end-to-end supervision and can create a mismatch between the network output and the corrected action. The third category embeds differentiable constraint enforcement layers into neural networks \citep{iftakher2026physicsinformed, liang2025efficient, lastrucci2025enforce, donti2021dc3, amos2017optnet}. These methods allow the learning objective to account for the corrected output and can provide stronger feasibility guarantees. However, enforcing general nonlinear inequality constraints remains challenging, particularly when the number of constraints is large and the projection must be computed efficiently.

\begin{figure*}[htbp]
  \centering
  \includegraphics[width=0.8\textwidth]{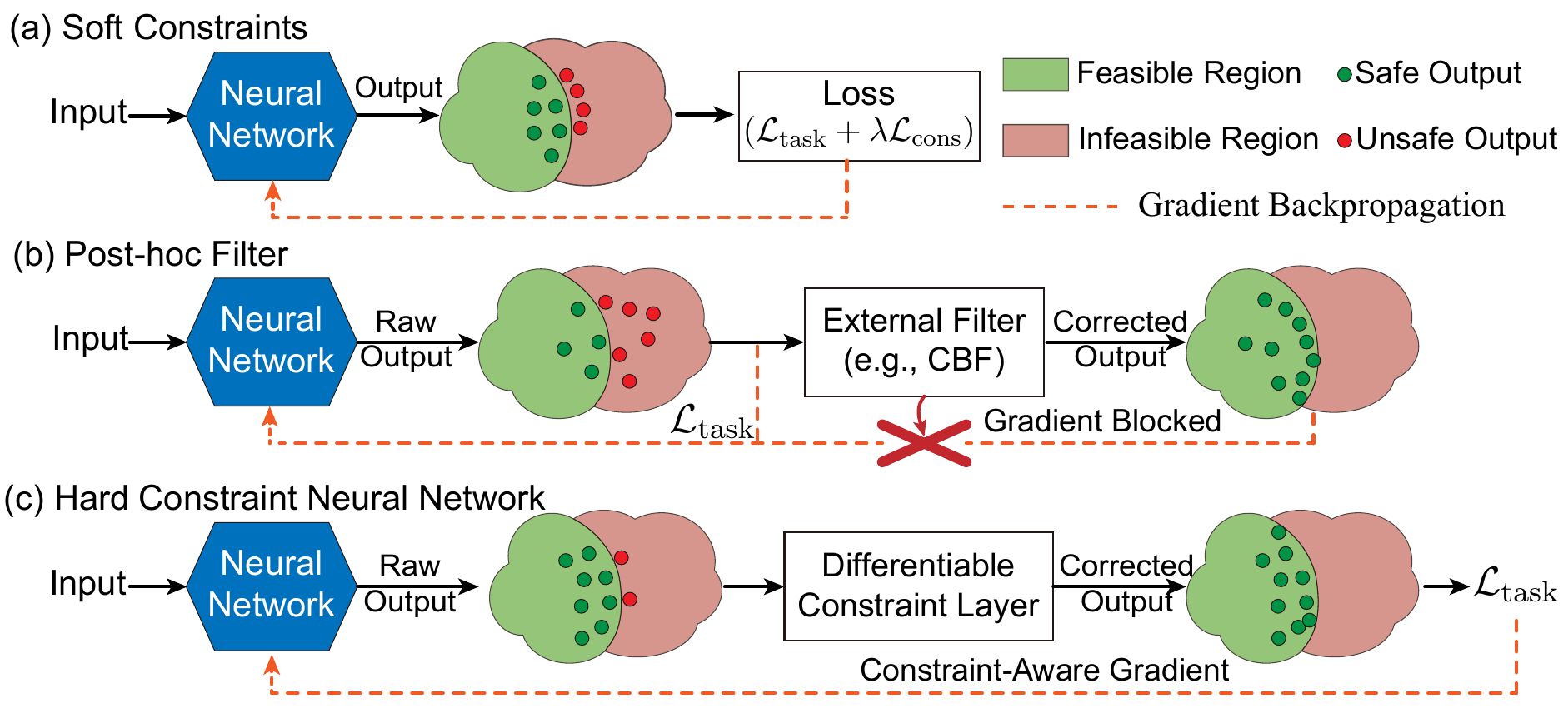}
  \caption{Comparison of three approaches to constraint handling in neural networks. (a) Soft constraints incorporate constraints as penalty terms in the loss function but provide no feasibility guarantee at inference time. (b) Post-hoc filters enforce constraints after prediction, which can create a mismatch between raw network outputs and corrected actions. (c) Hard constraint neural networks embed a differentiable constraint layer that enforces feasibility while preserving end-to-end gradient flow.}
  \label{fig:method}
\end{figure*}

To address these challenges, we propose DiffSlack, a differentiable hard constraint projection framework for high-dimensional nonlinear inequality constraints. The key idea is to reformulate inequalities as equalities with learnable slack variables. Unlike conventional auxiliary variables that are introduced only inside a solver, the slack variables in DiffSlack are predicted together with the task output. They therefore provide an amortized initialization of inactive constraint margins before projection. A damped Gauss-Newton projection layer then maps the augmented output onto the equality manifold defined by the slack reformulation. This layer is differentiable and can be inserted into the network for end-to-end training. To stabilize optimization, we further use a two-stage curriculum learning \citep{bengio2009curriculum}. The first stage trains the predictor with soft constraint losses and slack calibration, while the second stage activates the hard projection layer. This design allows the projection module to enforce feasibility while also providing task-aware correction signals during training.

We evaluate DiffSlack on constrained vehicle path planning, where the output trajectory must satisfy 200 coupled nonlinear inequality constraints from collision avoidance, curvature limits, and waypoint spacing. Our experiments compare DiffSlack with classical planners, soft-constrained learning, DC3 \citep{donti2021dc3}, and ENFORCE \citep{lastrucci2025enforce}. The results show that DiffSlack achieves a stronger trade-off among planning success, geometric constraint satisfaction, and computation time than existing learning-based methods. We further validate the generated trajectories in CARLA \citep{dosovitskiy2017carla} and on a physical vehicle.

The main contributions of this work are as follows:
\begin{enumerate}
  \item DiffSlack provides a scalable differentiable projection framework for neural prediction under nonlinear inequality constraints, with learnable slack variables serving as a data-driven initialization for iterative projection. Theoretical analysis characterizes the local convergence behavior and numerical conditioning.

  \item A two-stage curriculum strategy is proposed to stabilize end-to-end training with the hard projection layer, enabling the network to approach the feasible manifold before projection is activated.

  \item Experiments on large-scale constrained path planning with 200 nonlinear inequality constraints demonstrate that DiffSlack achieves higher success rates and better geometric constraint satisfaction than representative learning-based baselines. CARLA simulation and real-world vehicle tracking further validate the executability of the generated trajectories.
\end{enumerate}

\section{Related Work}

Soft constraint methods incorporate physical and operational requirements as penalty terms in the training objective, including penalty-based regularization \citep{csuzdi2026physicsinformed, yu2023knowledge, raissi2019physicsinformed} and Lagrangian duality methods \citep{fioretto2020lagrangian, park2023selfsupervised}. These approaches are simple and broadly applicable, but they do not guarantee constraint satisfaction at inference time. Their effectiveness depends on penalty weights, optimization quality, and the coverage of the training data. Post-hoc safety filters, such as control barrier functions \citep{zhang2025control, ames2017control} and predictive safety shields \citep{wabersich2023predictive, WABERSICH2021109597}, enforce constraints after prediction. They can provide strong deployment-time safety, but the correction is usually applied outside the prediction model. This may create a mismatch between the raw network output and the corrected action, especially when large corrections are frequently required. A more integrated strategy is to embed constraint satisfaction directly into the network architecture through differentiable enforcement layers \citep{jin2019survey}. Such methods allow the training objective to account for the corrected output and can provide stronger feasibility guarantees. We review them below according to the class of constraints they address.

\paragraph{Linear constraints}
A line of work exploits the algebraic structure of linear or affine constraints to design efficient enforcement layers with closed-form or iterative solutions. HCP \citep{chen2021theory} guarantees that network predictions satisfy finite-difference discretizations of governing PDEs via closed-form projection, but it is limited to linear equality constraints. LOOP-LC \citep{li2023learning} maps neural outputs from an \(\ell_\infty\)-norm unit ball to the feasible set through a gauge map, achieving iteration-free feasibility for hard linear equality and inequality constraints, but relying on linear structure and an available interior point. LinSATNet \citep{pmlr-v202-wang23at} proposes a Sinkhorn-based iterative scheme for positive linear constraints on input-independent outputs. GLinSAT \citep{guo2024glinsat} extends this framework to general linear constraints, with feasibility guarantees that tighten as the number of iterations increases. KKT-hPINN \citep{CHEN2024108764} embeds input-dependent linear equality constraints into physics-informed networks via KKT conditions and achieves strict feasibility through a closed-form projection, but it is confined to equality constraints. HardNet \citep{min2024hardnet} handles input-dependent affine equality and inequality constraints through a parallel projection layer with a closed-form forward pass, and establishes universal approximation guarantees under the constrained architecture.

\paragraph{Convex constraints}
Beyond methods that exploit linear or affine structure, another line of work formulates more general convex feasibility or optimization problems as differentiable layers. OptNet \citep{amos2017optnet} embeds a quadratic program into the network and computes gradients via implicit differentiation through the KKT conditions. CvxpyLayers \citep{cvxpylayers2019} generalizes this idea to convex programs expressible in CVXPY, enabling broader constraint coverage at the cost of solving an optimization problem during the forward pass. To reduce this overhead, RAYEN \citep{tordesillas2026rayen} guarantees feasibility by scaling the step length along a ray from a known interior point. It supports linear, convex quadratic, second-order cone, and linear matrix inequality constraints with low computational cost. PiNet \citep{grontas2026pinet} appends an orthogonal projection layer computed by Douglas-Rachford splitting, with gradients obtained through the implicit function theorem. It supports general convex constraints, including input-dependent ones.

\paragraph{General nonlinear constraints}
Extending hard constraint enforcement to general nonlinear constraints is more challenging. DC3 \citep{donti2021dc3} addresses nonlinear equality and inequality constraints by combining equality completion with gradient-based correction. However, its convergence can be sensitive to step size selection and may require many iterations for tight feasibility. BarrierNet \citep{wei2023barriernet} embeds differentiable control barrier functions into quadratic program layers for end-to-end safe learning, targeting systems whose safety constraints are defined through control-affine dynamics. ENFORCE \citep{lastrucci2025enforce} applies adaptive-depth Gauss-Newton projection with closed-form updates at each step and handles inequality constraints through a Fischer-Burmeister reformulation in an augmented multiplier space. KKT-HardNet \citep{iftakher2026physicsinformed} constructs the full KKT system of the projection problem and solves it with Newton or Gauss-Newton iterations, optionally after log-exponential transformations of nonlinear algebraic terms. Although this formulation can enforce nonlinear equality and inequality constraints to high precision, each inequality introduces additional slack variables, dual multipliers, and complementarity equations, leading to a large primal-dual nonlinear system. This overhead can become prohibitive for large-scale inequality-constrained prediction problems. Homeomorphic Projection \citep{liang2024homeomorphic} learns a minimum-distortion mapping between the constraint set and a unit ball via an invertible neural network, then recovers feasibility through bisection along the mapped radius. This provides feasibility guarantees for ball-homeomorphic constraint sets, but the optimality bound depends on the distortion of the learned mapping. Bisection Projection \citep{liang2025efficient} relaxes the topological requirement by training an auxiliary network to predict interior points of the feasible set, followed by bisection from these points to recover feasibility for general compact sets. Both methods introduce additional networks, which increase training complexity and may introduce approximation error. These limitations motivate a direct and scalable projection mechanism for nonlinear inequality-constrained neural prediction.

\section{Methodology}
\label{sec:method}

Let $f_\theta:\mathcal{X}\rightarrow\mathbb{R}^{N_O}$ be a neural network that maps an input $\mathbf{x}\in\mathcal{X}$ to a raw prediction $\hat{\mathbf{p}}=f_\theta(\mathbf{x})$. The prediction is required to satisfy $N_C$ input-dependent nonlinear inequality constraints
\begin{equation}
    \mathbf{g}(\mathbf{x},\mathbf{p})\leq \mathbf{0},
    \qquad
    \mathbf{g}:\mathcal{X}\times\mathbb{R}^{N_O}\rightarrow\mathbb{R}^{N_C}.
\end{equation}
The ideal correction is the nearest feasible point to the raw prediction:
\begin{equation}
    \mathbf{p}^{*}
    =
    \underset{\mathbf{p}}{\mathrm{argmin}}\;
    \frac{1}{2}\|\mathbf{p}-\hat{\mathbf{p}}\|_2^2
    \quad
    \mathrm{s.t.}\quad
    \mathbf{g}(\mathbf{x},\mathbf{p})\leq \mathbf{0}.
    \label{eq:original_problem}
\end{equation}
During projection, $\mathbf{x}$ is fixed. We therefore write $\mathbf{g}(\mathbf{p})$ for brevity, while all input-dependent quantities are treated as constants with respect to $\mathbf{p}$.

\begin{assumption}[Feasibility]
\label{ass:feasibility}
For every $\mathbf{x}\in\mathcal{X}$, the feasible set $\{\mathbf{p}\in\mathbb{R}^{N_O}\mid \mathbf{g}(\mathbf{x},\mathbf{p})\leq\mathbf{0}\}$ is non-empty.
\end{assumption}

DiffSlack embeds the correction in Eq.~(\ref{eq:original_problem}) as a differentiable projection layer. The main idea is to convert inequalities into an augmented equality manifold with learnable slack variables, and then project raw predictions onto this manifold through a Gauss-Newton layer. Figure~\ref{fig:framework} illustrates the overall framework.

\begin{figure*}
  \centering
  \includegraphics[width=\textwidth]{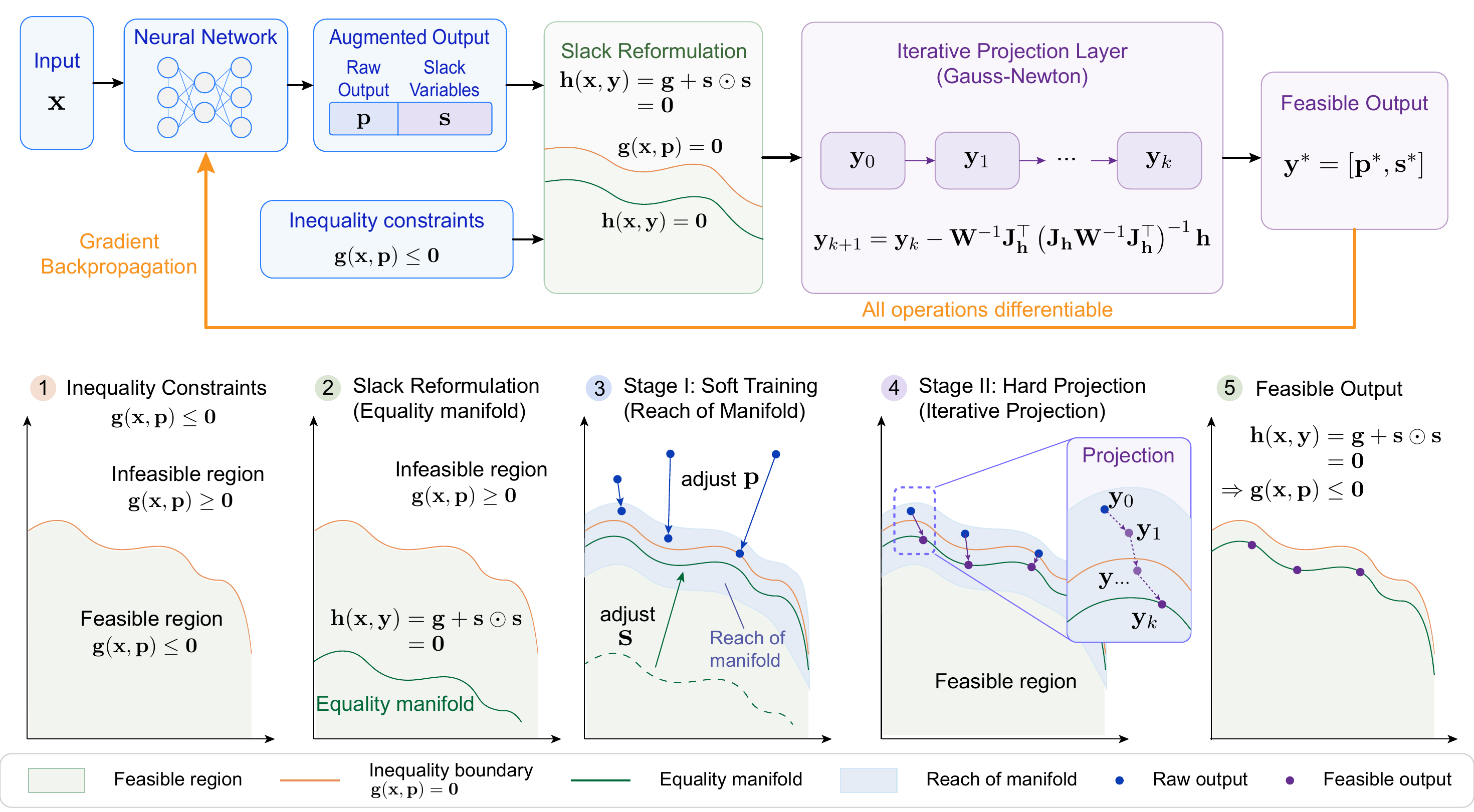}
  \caption{Overall framework and geometric intuition of DiffSlack. The network predicts both the task output $\hat{\mathbf{p}}$ and slack variables $\hat{\mathbf{s}}$. Inequality constraints are reformulated as an equality manifold $\mathcal{M}$ through $\mathbf{h}(\mathbf{y})=\mathbf{0}$. Stage I learns slack variables that approximate the inactive-side constraint margin and moves raw predictions toward the reach tube of $\mathcal{M}$. Stage II activates the differentiable Gauss-Newton projection layer, producing feasible outputs while preserving end-to-end gradient flow.}
  \label{fig:framework}
\end{figure*}

\subsection{Slack Variable Reformulation}
\label{sec:slack}

Inequality constraints define a feasible region rather than a smooth equality manifold. To obtain a differentiable projection target, we introduce slack variables $\mathbf{s}\in\mathbb{R}^{N_C}$ and define the augmented state
\begin{equation}
    \mathbf{y}=\begin{bmatrix}\mathbf{p}^\top & \mathbf{s}^\top\end{bmatrix}^\top
    \in\mathbb{R}^{N_O+N_C}.
\end{equation}
Each inequality is reformulated as
\begin{equation}
    \mathbf{h}(\mathbf{y})
    =
    \mathbf{g}(\mathbf{p})+\mathbf{s}\odot\mathbf{s}
    =\mathbf{0},
    \label{eq:slack_reformulation}
\end{equation}
where $\odot$ denotes the Hadamard product. Since $\mathbf{s}\odot\mathbf{s}$ is elementwise nonnegative, feasibility on the augmented manifold implies feasibility of the original inequalities.

\begin{theorem}[Inequality satisfaction]
\label{thm:strict_ineq}
For any $\mathbf{y}=[\mathbf{p}^\top,\mathbf{s}^\top]^\top$, if $\mathbf{h}(\mathbf{y})=\mathbf{0}$, then $\mathbf{g}(\mathbf{p})\leq\mathbf{0}$.
\end{theorem}

\begin{proof}
From Eq.~(\ref{eq:slack_reformulation}), $\mathbf{h}(\mathbf{y})=\mathbf{0}$ gives $\mathbf{g}(\mathbf{p})=-\mathbf{s}\odot\mathbf{s}$. Since $\mathbf{s}\odot\mathbf{s}\geq\mathbf{0}$ elementwise, the result follows.
\end{proof}

The augmented Jacobian is
\begin{equation}
    \mathbf{J}_{\mathbf{h}}(\mathbf{y})
    =
    \left[
        \mathbf{J}_{\mathbf{g}}(\mathbf{p})
        \;\Big|\;
        \mathrm{diag}(2\mathbf{s})
    \right],
    \label{eq:jacobian_block}
\end{equation}
where $\mathbf{J}_{\mathbf{g}}(\mathbf{p})\in\mathbb{R}^{N_C\times N_O}$. The diagonal slack block provides an independent correction direction for each inequality constraint in the augmented space. This structure improves the local conditioning of the linearized projection system whenever the corresponding slack variables are nonzero, while near-degenerate cases are stabilized by the damped solve in Section~\ref{sec:projection_layer}.

\subsection{Learned Slack as Margin Prediction}
\label{sec:learned_slack}

The slack variables in DiffSlack are not passive solver variables. They are predicted by the network and serve as a data-driven estimate of the inactive-side constraint margin. For a raw prediction \(\hat{\mathbf{p}}\), define
\begin{equation}
    \mathbf{q}(\hat{\mathbf{p}})
    :=
    [-\mathbf{g}(\hat{\mathbf{p}})]_+,
    \qquad
    [\mathbf{v}]_+=(\max\{v_i,0\})_{i=1}^{N_C}.
    \label{eq:slack_target}
\end{equation}
Here, \([\mathbf{g}(\hat{\mathbf{p}})]_+\) denotes active constraint violation, while \(\mathbf{q}(\hat{\mathbf{p}})\) denotes the margin of constraints that are already satisfied. Let \(\hat{\mathbf{q}}=\hat{\mathbf{s}}\odot\hat{\mathbf{s}}\). The equality residual can then be decomposed as
\begin{equation}
    \mathbf{h}(\hat{\mathbf{y}})
    =
    \mathbf{g}(\hat{\mathbf{p}})+\hat{\mathbf{q}}
    =
    [\mathbf{g}(\hat{\mathbf{p}})]_+
    +
    \big(
    \hat{\mathbf{q}}-\mathbf{q}(\hat{\mathbf{p}})
    \big).
    \label{eq:residual_decomposition}
\end{equation}
This decomposition separates the residual into active constraint violation and slack estimation error. With zero slack initialization, the equality residual also contains the inactive-margin term \(-\mathbf{q}(\hat{\mathbf{p}})\), even for constraints that are already satisfied. In contrast, when the predicted squared slack \(\hat{\mathbf{q}}\) approximates \(\mathbf{q}(\hat{\mathbf{p}})\), this inactive-margin component is absorbed before projection.

To quantify this effect, suppose the learned squared slack \(\hat{\mathbf{q}}_L\) approximates \(\mathbf{q}(\hat{\mathbf{p}})\) with mean-square error \(\varepsilon_L^2\), and the estimation error is conditionally zero-mean given \(\hat{\mathbf{p}}\):
\begin{equation}
    \mathbb{E}
    \left[
    \|\hat{\mathbf{q}}_L-\mathbf{q}(\hat{\mathbf{p}})\|_2^2
    \right]
    =
    \varepsilon_L^2 .
    \label{eq:slack_error}
\end{equation}
Taking the squared norm of Eq.~(\ref{eq:residual_decomposition}) and using the zero-mean assumption to remove the cross term gives
\begin{equation}
    \mathbb{E}
    \left[
    \|\mathbf{h}(\hat{\mathbf{y}}_L)\|_2^2
    \right]
    =
    \mathbb{E}
    \left[
    \|[\mathbf{g}(\hat{\mathbf{p}})]_+\|_2^2
    \right]
    +
    \varepsilon_L^2 .
    \label{eq:expected_residual_compact}
\end{equation}
For zero slack initialization \(\hat{\mathbf{q}}_0=\mathbf{0}\), the residual becomes \(\mathbf{h}(\hat{\mathbf{y}}_0)=\mathbf{g}(\hat{\mathbf{p}})\). Since \([\mathbf{g}]_+\) and \([-\mathbf{g}]_+\) have disjoint elementwise support,
\begin{equation}
    \mathbb{E}
    \left[
    \|\mathbf{h}(\hat{\mathbf{y}}_0)\|_2^2
    \right]
    =
    \mathbb{E}
    \left[
    \|[\mathbf{g}(\hat{\mathbf{p}})]_+\|_2^2
    \right]
    +
    \mathbb{E}
    \left[
    \|[-\mathbf{g}(\hat{\mathbf{p}})]_+\|_2^2
    \right].
    \label{eq:zero_slack_residual_compact}
\end{equation}
Therefore, learned slack reduces the expected equality residual whenever its estimation error is smaller than the inactive-margin energy.

Let
$
    \mathcal{M}
    =
    \{
    \mathbf{y}\in\mathbb{R}^{N_O+N_C}
    \mid
    \mathbf{h}(\mathbf{y})=\mathbf{0}
    \}
$
be the augmented feasible manifold. We consider a regular local tube around \(\mathcal{M}\), where \(\mathbf{h}\) is twice continuously differentiable, \(\mathbf{J}_{\mathbf{h}}\) is Lipschitz continuous, and its smallest singular value is uniformly bounded below by a constant \(\sigma_->0\):
\begin{equation}
    \sigma_{\min}
    \left(
    \mathbf{J}_{\mathbf{h}}(\mathbf{y})
    \right)
    \geq
    \sigma_- .
    \label{eq:regularity_bound}
\end{equation}
Under this standard local regularity condition, the residual controls the Euclidean distance to the manifold:
\begin{equation}
    \operatorname{dist}(\hat{\mathbf{y}},\mathcal{M})
    \leq
    2\sigma_-^{-1}
    \|\mathbf{h}(\hat{\mathbf{y}})\|_2 .
    \label{eq:distance_residual_bound}
\end{equation}
This follows from a first-order Taylor expansion of \(\mathbf{h}\) around the nearest point on \(\mathcal{M}\), where the second-order remainder is dominated in a sufficiently small neighborhood. Combining Eq.~(\ref{eq:expected_residual_compact}) and Eq.~(\ref{eq:distance_residual_bound}) gives
\begin{equation}
    \mathbb{E}
    \left[
    \operatorname{dist}(\hat{\mathbf{y}}_L,\mathcal{M})^2
    \right]
    \leq
    4\sigma_-^{-2}
    \left(
    \mathbb{E}
    \left[
    \|[\mathbf{g}(\hat{\mathbf{p}})]_+\|_2^2
    \right]
    +
    \varepsilon_L^2
    \right).
    \label{eq:learned_distance_bound}
\end{equation}
Hence, learned slack variables provide a principled mechanism for reducing the expected local distance to the augmented feasible manifold by estimating inactive margins before projection.

\subsection{Differentiable Gauss-Newton Projection Layer}
\label{sec:projection_layer}

Given a raw augmented prediction $\hat{\mathbf{y}}=[\hat{\mathbf{p}}^\top,\hat{\mathbf{s}}^\top]^\top$, DiffSlack projects it onto $\mathcal{M}$ by solving
\begin{equation}
    \mathbf{y}^{*}
    =
    \underset{\mathbf{y}}{\mathrm{argmin}}\;
    \frac{1}{2}\|\mathbf{y}-\hat{\mathbf{y}}\|_{\mathbf{W}}^2
    \quad
    \mathrm{s.t.}\quad
    \mathbf{h}(\mathbf{y})=\mathbf{0},
    \label{eq:projection_problem}
\end{equation}
where
\begin{equation}
    \|\mathbf{v}\|_{\mathbf{W}}^2=\mathbf{v}^\top\mathbf{W}\mathbf{v},
    \qquad
    \mathbf{W}=\mathrm{diag}(w_o\mathbf{1}_{N_O},w_s\mathbf{1}_{N_C}),
    \qquad
    w_o>w_s>0.
\end{equation}
The larger task-output weight discourages unnecessary changes to the original prediction coordinates, while the smaller slack weight allows residual mismatch to be absorbed by the auxiliary slack variables whenever possible.

\paragraph{Reach tube.}
The projection layer is local. Let $\operatorname{reach}(\mathcal{M})$ denote the largest radius such that every point within that distance from $\mathcal{M}$ has a unique nearest point on $\mathcal{M}$ \citep{federer1959curvature}. For any $\rho<\operatorname{reach}(\mathcal{M})$, define the reach tube
\begin{equation}
    \mathcal{T}_\rho
    =
    \{\mathbf{y}\mid \operatorname{dist}(\mathbf{y},\mathcal{M})<\rho\}.
    \label{eq:reach_tube}
\end{equation}
Inside this tube, the nearest-point projection is geometrically well-defined. Under the local regularity condition in Eq.~(\ref{eq:regularity_bound}), a sufficiently small sub-tube also forms a local convergence basin for the Gauss-Newton projection. Therefore, Stage I is designed to move raw outputs toward this reach tube before the hard projection layer is activated. Since
\begin{equation}
    \mathbb{P}(\hat{\mathbf{y}}\in\mathcal{T}_\rho)
    \geq
    1-
    \mathbb{E}[\operatorname{dist}_{\mathbf{W}}(\hat{\mathbf{y}},\mathcal{M})^2]/\rho^2,
    \label{eq:markov_reach}
\end{equation}
reducing the expected distance to the manifold increases this distribution-free lower bound on the probability of entering the reach tube.

\paragraph{Gauss-Newton update.}
Solving Eq.~(\ref{eq:projection_problem}) exactly would require second-order constraint derivatives. We instead linearize $\mathbf{h}$ around the current iterate $\mathbf{y}_k$ and solve a weighted minimum-norm correction problem:
\begin{equation}
    \mathbf{y}_{k+1}
    =
    \underset{\mathbf{y}}{\mathrm{argmin}}\;
    \frac{1}{2}\|\mathbf{y}-\mathbf{y}_k\|_{\mathbf{W}}^2
    \quad
    \mathrm{s.t.}\quad
    \mathbf{h}(\mathbf{y}_k)
    +
    \mathbf{J}_{\mathbf{h}}(\mathbf{y}_k)(\mathbf{y}-\mathbf{y}_k)
    =
    \mathbf{0}.
    \label{eq:linearized_projection}
\end{equation}
The closed-form update is
\begin{equation}
    \mathbf{y}_{k+1}
    =
    \mathbf{y}_k
    -
    \mathbf{W}^{-1}\mathbf{J}_{\mathbf{h}}^\top(\mathbf{y}_k)
    \left(
    \mathbf{J}_{\mathbf{h}}(\mathbf{y}_k)
    \mathbf{W}^{-1}
    \mathbf{J}_{\mathbf{h}}^\top(\mathbf{y}_k)
    \right)^{-1}
    \mathbf{h}(\mathbf{y}_k).
    \label{eq:newton_step}
\end{equation}
The Gram matrix is solved by Cholesky decomposition with diagonal damping $\delta\mathbf{I}$, where $\delta=10^{-4}$ in our experiments. Iteration stops when $\|\mathbf{h}(\mathbf{y}_k)\|_\infty<\epsilon$ or the number of iterations reaches $I_{\max}$.

\paragraph{Local convergence and conditioning.}
Around a feasible point $\mathbf{y}^{*}\in\mathcal{M}$ satisfying the regularity condition, the update in Eq.~(\ref{eq:newton_step}) satisfies
\begin{equation}
    \|\mathbf{e}^{(k+1)}\|_{\mathbf{W}}
    \leq
    \|\mathbf{M}_{\mathbf{W}}\|_{\mathbf{W}}
    \|\mathbf{e}^{(k)}\|_{\mathbf{W}}
    +
    O(\|\mathbf{e}^{(k)}\|_{\mathbf{W}}^2),
    \label{eq:local_linear_convergence}
\end{equation}
where $\mathbf{e}^{(k)}=\mathbf{y}_k-\mathbf{y}^{*}$ and
\begin{equation}
    \mathbf{M}_{\mathbf{W}}
    =
    \mathbf{I}
    -
    \mathbf{W}^{-1}\mathbf{J}_{\mathbf{h}}^\top
    \left(
    \mathbf{J}_{\mathbf{h}}
    \mathbf{W}^{-1}
    \mathbf{J}_{\mathbf{h}}^\top
    \right)^{-1}
    \mathbf{J}_{\mathbf{h}}
    \label{eq:mw_def}
\end{equation}
is the $\mathbf{W}$-weighted tangent projector. Since $\mathbf{M}_{\mathbf{W}}$ is a $\mathbf{W}$-orthogonal tangent projector, the linearized update suppresses the normal component of the local error to first order.

The slack block also stabilizes the linear solve. From Eq.~(\ref{eq:jacobian_block}), the Gram matrix is
\begin{equation}
    \mathbf{J}_{\mathbf{h}}\mathbf{W}^{-1}\mathbf{J}_{\mathbf{h}}^\top
    =
    w_o^{-1}\mathbf{J}_{\mathbf{g}}\mathbf{J}_{\mathbf{g}}^\top
    +
    4w_s^{-1}\mathrm{diag}(\mathbf{s}^2).
    \label{eq:gram_slack}
\end{equation}
Thus nonzero slack components lift the spectral floor:
\begin{equation}
    \lambda_{\min}\!\left(
    \mathbf{J}_{\mathbf{h}}\mathbf{W}^{-1}\mathbf{J}_{\mathbf{h}}^\top
    \right)
    \geq
    w_o^{-1}\lambda_{\min}\!\left(\mathbf{J}_{\mathbf{g}}\mathbf{J}_{\mathbf{g}}^\top\right)
    +
    4w_s^{-1}\min_i s_i^2.
    \label{eq:spectral_floor}
\end{equation}
This reduces residual amplification in the linear solve, especially when $\mathbf{J}_{\mathbf{g}}$ is nearly rank-deficient.

\subsection{Gradient Flow and Implicit Constraint Supervision}
\label{sec:gradient_flow}

The projection layer is differentiable. Instead of backpropagating through all solver iterations, we differentiate the KKT conditions of Eq.~(\ref{eq:projection_problem}) at the projected point, following implicit differentiation for optimization layers \citep{gould2016differentiating, blondel2022efficient, quentin2022implicit}. Locally, this gives the first-order sensitivity
\begin{equation}
    \partial\mathbf{y}^{*}/\partial\hat{\mathbf{y}}
    =
    \mathbf{M}_{\mathbf{W}}.
    \label{eq:projection_jacobian}
\end{equation}
Therefore, gradients passed to the network are projected onto the tangent space of $\mathcal{M}$ at $\mathbf{y}^{*}$ under the $\mathbf{W}$-weighted geometry. Since $\mathbf{M}_{\mathbf{W}}$ is a projector with spectral radius no larger than one, the projection step does not locally amplify gradients. More importantly, gradient components normal to the constraint manifold are suppressed, while tangent components are preserved. The network is therefore trained through directions that are locally compatible with the constraint geometry. We refer to this effect as implicit constraint supervision: even when the external task supervision is coarse, the projection layer filters the output gradient toward feasible directions.

\subsection{Two-Stage Curriculum Training}
\label{sec:curriculum}

Directly activating the hard projection layer from random initialization is unstable, because Gauss-Newton projection is local and assumes that raw augmented outputs are close to the reach tube of $\mathcal{M}$. DiffSlack therefore uses a two-stage curriculum learning \citep{bengio2009curriculum}.

\paragraph{Stage I: slack-calibrated soft training.}
In Stage I, the projection layer is disabled. The network is trained with
\begin{equation}
    \mathcal{L}_{\mathrm{Stage1}}
    =
    \mathcal{L}_{\mathrm{task}}
    +
    \lambda_{\mathrm{soft}}\mathcal{L}_{\mathrm{soft}}
    +
    \lambda_{\mathrm{slack}}\mathcal{L}_{\mathrm{slack}}.
    \label{eq:stage1_loss}
\end{equation}
The soft constraint loss penalizes active inequality violations:
\begin{equation}
    \mathcal{L}_{\mathrm{soft}}
    =
    N_C^{-1}
    \left\|[\mathbf{g}(\hat{\mathbf{p}})]_+\right\|_1.
    \label{eq:soft_loss}
\end{equation}
The slack loss fits the squared slack variables to the inactive-side margin of the current prediction:
\begin{equation}
    \mathcal{L}_{\mathrm{slack}}
    =
    N_C^{-1}
    \left\|
    \mathbf{g}(\mathrm{sg}(\hat{\mathbf{p}}))
    +
    \hat{\mathbf{s}}\odot\hat{\mathbf{s}}
    \right\|_1,
    \label{eq:slack_loss}
\end{equation}
where $\mathrm{sg}(\cdot)$ denotes the stop-gradient operator. The stop-gradient prevents $\mathcal{L}_{\mathrm{slack}}$ from altering the task prediction and ensures that this term only calibrates the slack variables. By Eq.~(\ref{eq:residual_decomposition}), $\mathcal{L}_{\mathrm{soft}}$ reduces the active-violation term, while $\mathcal{L}_{\mathrm{slack}}$ reduces the slack estimation error. Together, they reduce the distance from raw augmented outputs to $\mathcal{M}$ and increase the probability that samples lie inside the local reach tube.

\paragraph{Stage II: end-to-end hard projection training.}
In Stage II, the projection layer is activated. The network produces $\hat{\mathbf{y}}$, the projection layer returns $\mathbf{y}^{*}=\mathcal{P}(\hat{\mathbf{y}})$, and the task loss is computed on the projected output. The slack loss is removed because feasibility is now enforced by the projection layer. The objective becomes
\begin{equation}
    \mathcal{L}_{\mathrm{Stage2}}
    =
    \mathcal{L}_{\mathrm{task}}(\mathbf{y}^{*})
    +
    \lambda_{\mathrm{proj}}\mathcal{L}_{\mathrm{proj}}
    +
    \lambda_{\mathrm{soft}}\mathcal{L}_{\mathrm{soft}}.
    \label{eq:stage2_loss}
\end{equation}
The projection loss is
\begin{equation}
    \mathcal{L}_{\mathrm{proj}}
    =
    \|\hat{\mathbf{y}}-\mathbf{y}^{*}\|_2^2.
    \label{eq:projection_loss}
\end{equation}
It encourages raw predictions to remain close to the feasible manifold, which stabilizes the subsequent projection. The soft constraint loss is retained as an auxiliary regularizer that discourages large raw violations before projection.

Overall, Stage I calibrates slack variables to estimate inactive-side margins and moves raw outputs toward the reach tube. Stage II then imposes hard feasibility through differentiable projection and propagates constraint-aware gradients back to the network.

\section{EXPERIMENTS}

\subsection{Problem Formulation}\label{sec:problem}

We consider the problem of generating a kinematically feasible and collision-free path for a vehicle navigating through obstacle-dense unstructured environments \citep{guo2024survey}. The path is represented as a sequence of $T=40$ discrete waypoints $\mathbf{P} = (\mathbf{p}_1, \dots, \mathbf{p}_T) \in \mathbb{R}^{T \times 2}$, where each $\mathbf{p}_t = (x_t, y_t)$ denotes a 2D position in the workspace. The vehicle starts from a fixed initial position $\mathbf{p}_0 = (0, 0)$. The planning objective is to minimize the total path length while satisfying a set of hard constraints:
\begin{equation}
\min_{\mathbf{P}} \sum_{t=0}^{T-1} \|\mathbf{p}_{t+1} - \mathbf{p}_t\|_2 
\quad \text{s.t.} \quad \mathcal{C}_\text{col}(\mathbf{P}) \leq 0, \quad 
\mathcal{C}_\text{kin}(\mathbf{P}) \leq 0, \quad 
\mathcal{C}_\text{spc}(\mathbf{P}) \leq 0.
\end{equation}

The three constraint categories are defined as follows.

\paragraph{Collision Avoidance Constraint}
The vehicle body is approximated by three overlapping circles centered at positions $\{\mathbf{q}_{t,r}\}_{r=1}^{3}$ derived from each waypoint $\mathbf{p}_t$. Each obstacle is modeled as a convex polygon described by $M_j$ half-plane inequalities $l_{j,m}(\mathbf{p}) \geq 0$. To obtain a differentiable signed distance for each obstacle, we apply the 
Log-Sum-Exp (LSE) approximation \citep{boyd2004convex}:
\begin{equation}
c_j(\mathbf{p}) \approx -\alpha^{-1}\log\!\left(\sum_{m=1}^{M_j} 
\exp\!\left(-\alpha\, l_{j,m}(\mathbf{p})\right)\right),
\end{equation}
where $c_j(\mathbf{p}) < 0$ indicates that $\mathbf{p}$ lies inside obstacle $j$, and $\alpha = 10$ controls the approximation tightness. The collision constraint then requires that every circle $r$ at waypoint $t$ lies outside all obstacles simultaneously, which is enforced by a second LSE aggregation over the $N_\text{obs}$ obstacles:
\begin{equation}
\mathcal{C}_\text{col}^{t,r}(\mathbf{p}_t) = 
\alpha^{-1}\log\!\left(\sum_{j=1}^{N_\text{obs}} 
\exp\!\left(\alpha\, c_j(\mathbf{q}_{t,r})\right)\right) \leq 0, 
\end{equation}
yielding $3T$ collision constraints in total. At evaluation time, collision detection is performed using exact rectangle-obstacle intersection checks to ensure accurate assessment of the success rate.

\paragraph{Curvature Constraint}
The vehicle is modeled by the kinematic bicycle model \citep{polack2017kinematic}, which imposes a minimum turning radius 
$R_\text{min}$ determined by the maximum steering angle. The curvature constraint requires that the local curvature $\kappa_t$ at each waypoint does not exceed the maximum allowable value $\kappa_\text{max} = 1/R_\text{min}$:
\begin{equation}
\mathcal{C}_\text{kin}^{t}(\mathbf{P}) = \kappa_t - \kappa_\text{max} 
\leq 0, \quad \forall t \in \{1,\dots,T\},
\end{equation}
yielding $T$ curvature constraints in total.

\paragraph{Waypoint Spacing Constraint}
To ensure geometric regularity and suppress interpolation errors during path tracking, the Euclidean distance between consecutive waypoints is constrained not to exceed a threshold of $d_\text{max} = 1.0\,\unit{m}$:
\begin{equation}
\mathcal{C}_\text{spc}^{t}(\mathbf{P}) = \|\mathbf{p}_{t+1} - 
\mathbf{p}_t\|_2 - d_\text{max} \leq 0, \quad 
\forall t \in \{0,\dots,T-1\}.
\end{equation}

Altogether, the three constraint categories produce $200$ inequality constraints for $T = 40$ waypoints.

\subsection{Environment Setup}
\paragraph{Environment and Data}
The experimental environment is spatially partitioned into an obstacle generation zone ($x \in [4, 28]\,\unit{m}$, $y \in [-10, 10]\,\unit{m}$) and a goal sampling zone ($x \in [30, 34]\,\unit{m}$, $y \in [-8, 8]\,\unit{m}$). Each scenario contains $N_\text{obs} = 8$ randomly generated quadrilateral obstacles with side lengths sampled from $l \sim \mathcal{U}(1, 4)\,\unit{m}$, as illustrated in Fig.~\ref{fig:env} (a). A safety margin is enforced around each 
obstacle during generation to ensure the existence of at least one feasible path in every scenario. A total of 200,000 scenarios are generated and partitioned into training, validation, and testing sets at a ratio of 6:3:1. 
\begin{figure*}
  \centering
  \includegraphics[width=\textwidth]{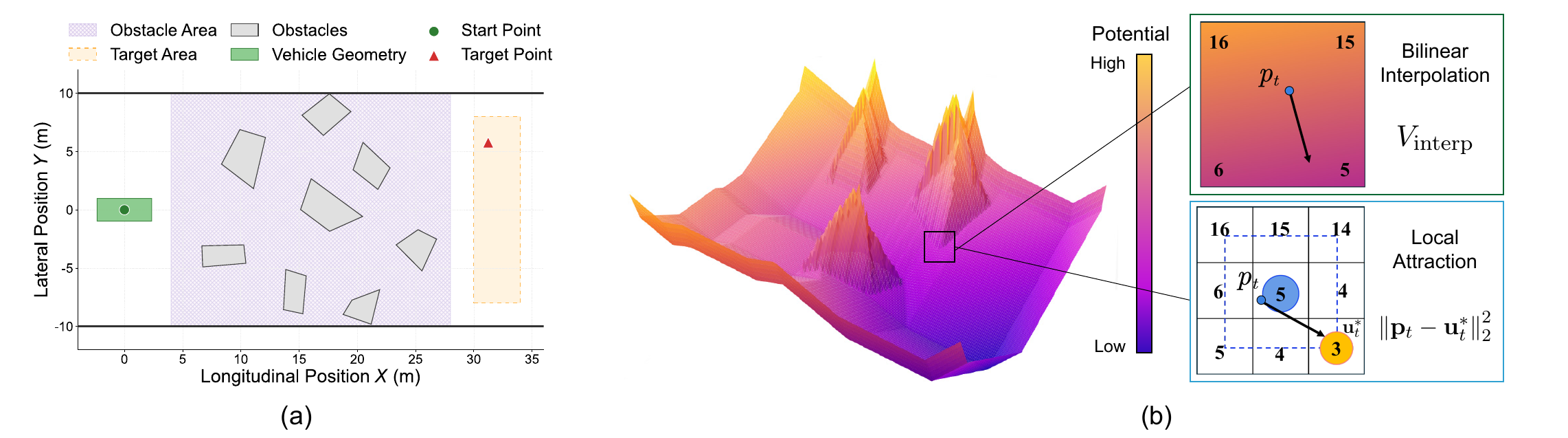}
  \caption{Experimental environment and G-APF supervision construction. (a) Static path-planning scenario with randomly generated convex polygonal obstacles and a sampled target area. (b) Global-guided artificial potential field, where high values are assigned near obstacles and low values near the target. For each waypoint, bilinear interpolation provides the potential value, and the local attraction target is selected from the lowest-potential point in the local neighborhood.
}\label{fig:env}
\end{figure*}

\paragraph{Supervision Signal}
Rather than relying on kinematically feasible expert demonstrations, which require expensive online solvers such as NMPC, we construct a Global-guided Artificial Potential Field (G-APF)~\citep{khatib1986realtime} as a cost-effective training signal. A global reference path $\mathcal{P}^*$ is first computed via Dijkstra's algorithm \citep{dijkstra1959note}, and a dense potential field $V(\mathbf{u})$ is propagated from $\mathcal{P}^*$ into the free space using multi-source wavefront expansion, with high penalties imposed within obstacle regions to create repulsive gradients at obstacle boundaries. The resulting field forms a convex valley along $\mathcal{P}^*$, providing dense gradient guidance across the workspace. The supervision loss combines a potential energy term that aligns each waypoint with the global energy valley, and a local attraction term that prevents stagnation in obstacle-dense regions where the potential gradient may become nearly orthogonal to the desired direction of 
progress:
\begin{equation}
\mathcal{L}_\text{pot} = T^{-1}\sum_{t=1}^{T}\left(
V_\text{interp}(\mathbf{p}_t) + \beta\min_{\mathbf{u}\in
\mathcal{N}(\mathbf{n}_t)}\|\mathbf{p}_t - \mathbf{u}\|_2^2\right).
\end{equation}
Compared to expert demonstrations, the G-APF provides only global topological guidance rather than precise waypoint supervision, making it a coarser but significantly cheaper training signal. The constructed potential field is illustrated in Fig.~\ref{fig:env} (b).


\paragraph{Baselines}
We compare the proposed framework against six representative baselines spanning classical and learning-based planning paradigms, with implementation details provided in Appendix~\ref{app:baselines}. (1) \textbf{Hybrid A*} \citep{HybridAStar} is a search-based planner that employs Reeds-Shepp curves to ensure kinematic feasibility. (2) \textbf{Informed RRT*} \citep{gammell2014informed} is a sampling-based planner with obstacle inflation of \qty{1.2}{m}. (3) \textbf{NMPC} \citep{rawlings2017model} is an optimization-based planner implemented via the CasADi framework \citep{andersson2019casadi} and the IPOPT solver \citep{wachter2006implementation}, sharing the same kinematic model and obstacle representation as our method. (4) \textbf{IL (Pure)} is an imitation learning baseline trained directly on NMPC expert demonstrations without any constraint handling. (5) \textbf{IL + Soft} extends IL (Pure) by incorporating soft-constraint penalties during training. (6) \textbf{DC3} \citep{donti2021dc3} handles nonlinear equality and inequality constraints by applying gradient-based correction to reduce constraint violation. (7) \textbf{ENFORCE} \citep{lastrucci2025enforce} is a differentiable projection-based method using adaptive-depth Gauss-Newton iterations. For inequality constraints, ENFORCE adopts a Fischer-Burmeister reformulation in an augmented multiplier space. For a fair comparison, DC3, ENFORCE, and DiffSlack are implemented with the same policy network, input representation, and G-APF supervision signal. In all three methods, the gradients of the constraint functions are computed by automatic differentiation.

\paragraph{Metrics}
We evaluate performance using two sets of metrics.

\textbf{Static Planning Metrics} include: (1) Computation Time (CT), the average inference latency per scenario; (2) Average Path Length (APL), the mean total path length computed over collision-free paths only; (3) Average Goal Distance (AGD), the mean Euclidean distance from the 
final waypoint to the goal, computed over collision-free paths only; (4) Success Rate (SR), the percentage of trials in which the planner produces a valid collision-free path; (5) Kinematic Compliance ($\mathcal{S}_\text{kin}$), the mean ratio of maximum allowable curvature to local curvature:
\begin{equation}
\mathcal{S}_\text{kin} = T^{-1}\sum_{t=1}^{T}\min\left(\kappa_\text{max}/\kappa_t,\ 1\right);
\end{equation}
and (6) Spacing Compliance ($\mathcal{S}_\text{spc}$), the degree to which consecutive waypoint spacings satisfy the upper bound $d_\text{max}$:
\begin{equation}
\mathcal{S}_\text{spc} = 1 - d_\text{max}^{-1}\sum_{t=0}^{T-1}\max\left(\|\mathbf{p}_{t+1} - \mathbf{p}_t\|_2 - d_\text{max},\ 0\right).
\end{equation}
Metrics~(4),~(5), and~(6) respectively quantify the satisfaction of the three physical constraints embedded in the projection layer: collision avoidance, kinematic feasibility, and waypoint spacing.

\textbf{Dynamic Tracking Metrics} assess physical feasibility through tracking experiments in CARLA \citep{dosovitskiy2017carla} and real car scenarios. They include: (1) RMSE Cross-Track Error (CTE); (2) Average Heading Error (AHE); and (3) Steering Smoothness ($\mathcal{J}_\text{smooth}$):
\begin{equation}
\mathcal{J}_\text{smooth} = T^{-1}\sum_{t=0}^{T-1}|\delta_{t+1} - \delta_t|,
\end{equation}
where $\delta_t$ denotes the steering angle at time step $t$.

\begin{table}[htbp]
\centering
  \caption{Network Configurations and Hyperparameters}\label{tab:hyperparameters}
  \begin{tabular}{@{}ll@{}}
    \toprule
    Parameter & Value \\ \midrule
    Hidden Dimensions & [128, 256, 512, 512]  \\
    Activation Function & ReLU  \\
    Stage 1 training & 300 epochs, $lr=10^{-4}$, Batch Size 512  \\
    Stage 2 training & 100 epochs, $lr=3\times10^{-5}$, Batch Size 256  \\
    Tolerance ($\epsilon$) & $10^{-3}$  \\
    Max Iterations ($I_{\text{max}}$) & 50  \\
    Path Weight ($w_p$) & 5.0  \\
    Slack Weight ($w_s$) & 1.0  \\ \bottomrule
  \end{tabular}
\end{table}
\subsection{Implementation Details}\label{sec:implementation}
The specific configurations for the policy network, training curriculum, and inference solver are summarized in Table~\ref{tab:hyperparameters}. The policy network takes a 66-dimensional vector as input, comprising the target coordinates and the vertices of the obstacles.
All learning-based methods are trained for a total of 400 epochs on an Intel i9-14900KF CPU paired with an NVIDIA RTX 2080 Ti GPU. For inference, learning-based methods are evaluated on the NVIDIA RTX 2080 Ti GPU, while non-learning methods run on the Intel i9-14900KF CPU to reflect their respective deployment conditions.

\subsection{Comparative Analysis of Planning Performance}

\begin{figure*}[htbp]
  \centering
  \includegraphics[width=\textwidth]{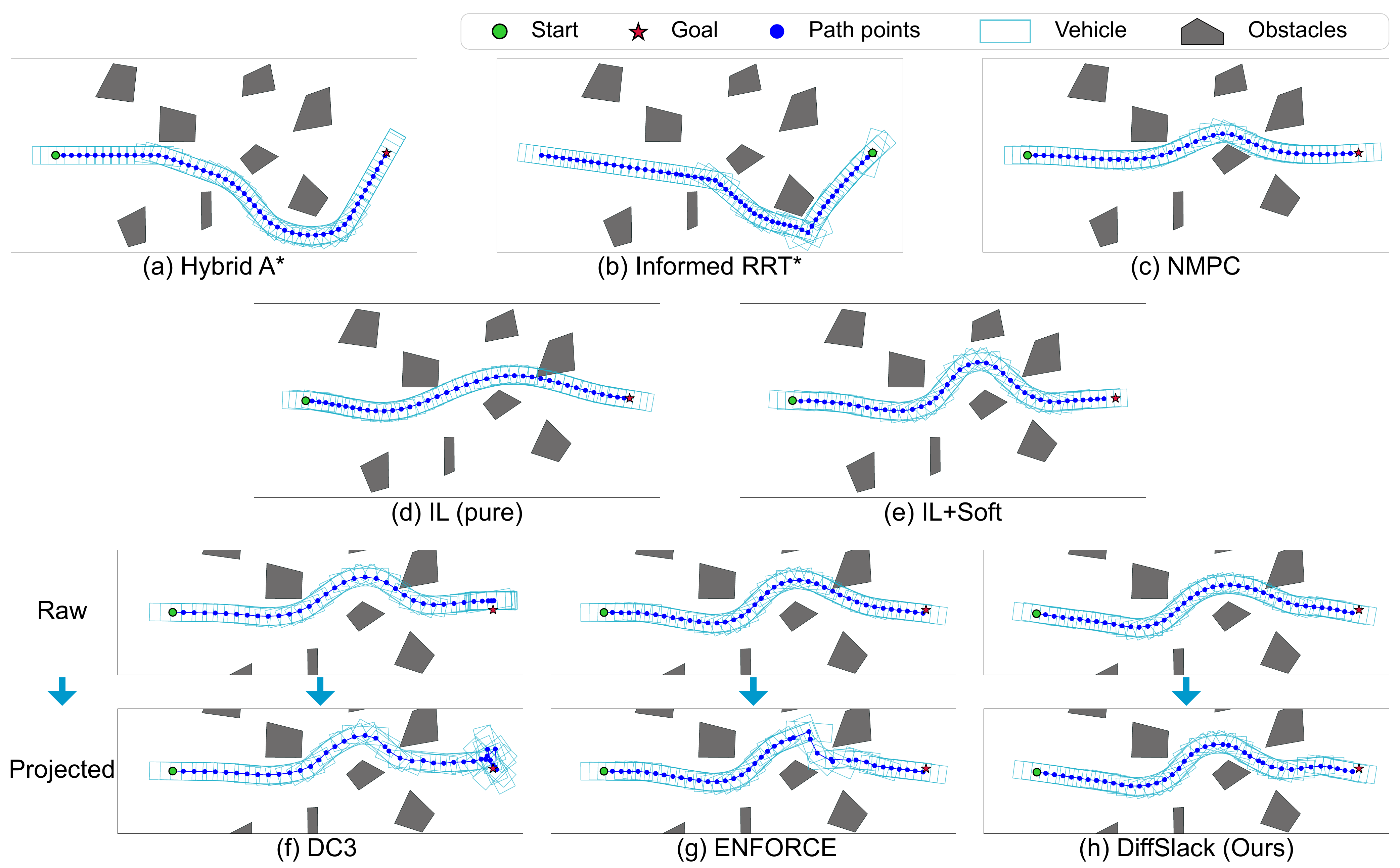}
  \caption{Qualitative comparison on a representative planning scenario. The upper panels show final trajectories produced by representative classical planners and learning-based baselines. The lower panels compare projection-based neural methods, where each column shows the raw prediction before projection and the projected trajectory after correction. DiffSlack yields a smoother projected trajectory in the obstacle-dense region, consistent with its higher planning success and geometric constraint satisfaction.}\label{fig:qualitative}
\end{figure*}

\begin{table*}[htbp]
  \centering
  \caption{Comparative Analysis of Path Planning Performance. $\uparrow$ ($\downarrow$) indicates that higher (lower) values are better. \textbf{Bold} values indicate the best results in each category.}\label{tab:comparison_static}
    \begin{tabular*}{\textwidth}{@{\extracolsep{\fill}}llcccccc}
      \toprule
      Category & Method & CT (s) $\downarrow$ & APL (m) $\downarrow$ & AGD (m) $\downarrow$ & SR (\%) $\uparrow$ & $\mathcal{S}_{\text{kin}}$(\%) $\uparrow$ & $\mathcal{S}_\text{spc}$ (\%) $\uparrow$ \\
      \midrule
      \multirow{3}{*}{Non-learning} 
        & Hybrid A*     & 12.3392   & 35.57 & 0.0 & \textbf{98.80} & \textbf{100.0} & 100.0 \\
        & Informed RRT* & 9.4949    & 34.59 & 0.0 & 86.26 & 96.50 & 100.0 \\
        & NMPC & \textbf{1.5308} & \textbf{34.47} & 0.0 & 82.35 & 100.0 & 100.0\\
      \midrule
      \multirow{5}{*}{Learning-based} 
        & IL (Pure) & \textbf{0.0002} & \textbf{33.16} & 0.7029 & 36.32 & \textbf{99.59} & 99.23 \\
        & IL+Soft & \textbf{0.0002} & 34.04 & 1.1161 & 72.66 & 98.60 & 99.85 \\
        & DC3($t=50$) & 0.1012 & 33.85 & 0.6927 & 82.84 & 97.06 & 92.61 \\
        & DC3($t=100$) & 0.2025 & 33.80 & 0.7718 & 86.44& 97.15 & 94.57 \\
        & ENFORCE($t=50$) & 0.1101 & 33.31 & 0.9698 & 84.05 & 97.28 & 99.29 \\
        & DiffSlack (Ours) & 0.0983 & 33.59 & \textbf{0.3474} & \textbf{93.80} & 99.48 & \textbf{99.89} \\
      \bottomrule
    \end{tabular*}
\end{table*}

Table~\ref{tab:comparison_static} and Fig.~\ref{fig:qualitative} report the quantitative and qualitative planning results on 20,000 test scenarios. Non-learning planners achieve strong constraint satisfaction because feasibility is explicitly considered during search or optimization. However, their computation time is much higher than that of learning-based methods. In contrast, pure imitation learning is extremely fast but has a low success rate and large goal deviation, showing that feed-forward prediction alone is insufficient for reliable constrained planning in this setting. Adding a soft constraint penalty improves feasibility, but the success rate remains clearly lower than that of hard-constrained neural methods.

Among projection-based learning methods, DC3 improves feasibility through iterative gradient correction, but its spacing satisfaction remains relatively low, especially under the smaller correction budget. Increasing its correction budget from \(t=50\) to \(t=100\) improves the success rate and spacing satisfaction, but almost doubles the computation time. ENFORCE provides a strong Gauss-Newton projection baseline and achieves computation time comparable to DiffSlack. Nevertheless, it still has lower success rate, lower curvature satisfaction, and lower spacing satisfaction than DiffSlack. These results suggest that DiffSlack achieves more reliable geometric correction under a comparable computation budget, which is consistent with the intended role of the learned slack representation.

Figure~\ref{fig:qualitative} further illustrates this behavior. The upper panels compare representative final trajectories, while the lower panels show raw and projected outputs of DC3, ENFORCE, and DiffSlack in the same scenario. DC3 starts from a different raw prediction and shows a sharper projected correction near the goal-side region. In contrast, ENFORCE and DiffSlack start from more comparable raw predictions, but their projected trajectories differ noticeably near the obstacle-dense region. ENFORCE exhibits a sharper local correction, whereas DiffSlack produces a smoother projected trajectory with more stable obstacle-avoidance behavior. This qualitative result is consistent with the higher planning success and geometric constraint satisfaction reported in Table~\ref{tab:comparison_static}.

\subsection{Dynamic Validation and Path Tracking}

\begin{figure*}[htbp]
  \centering
  \includegraphics[width=0.9\textwidth]{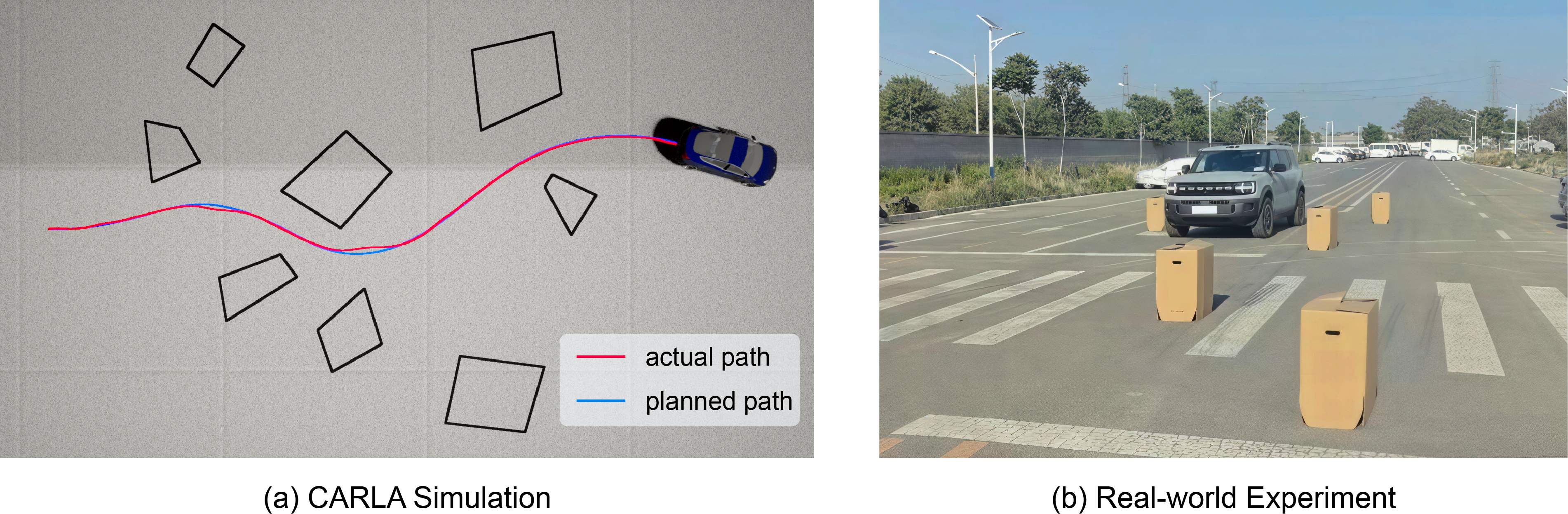}
  \caption{Dynamic validation environments. (a) CARLA simulation. (b) Real-world experiment with physical obstacles.}\label{fig:experiment}
\end{figure*}

\begin{figure*}
  \centering
  \includegraphics[width=\textwidth]{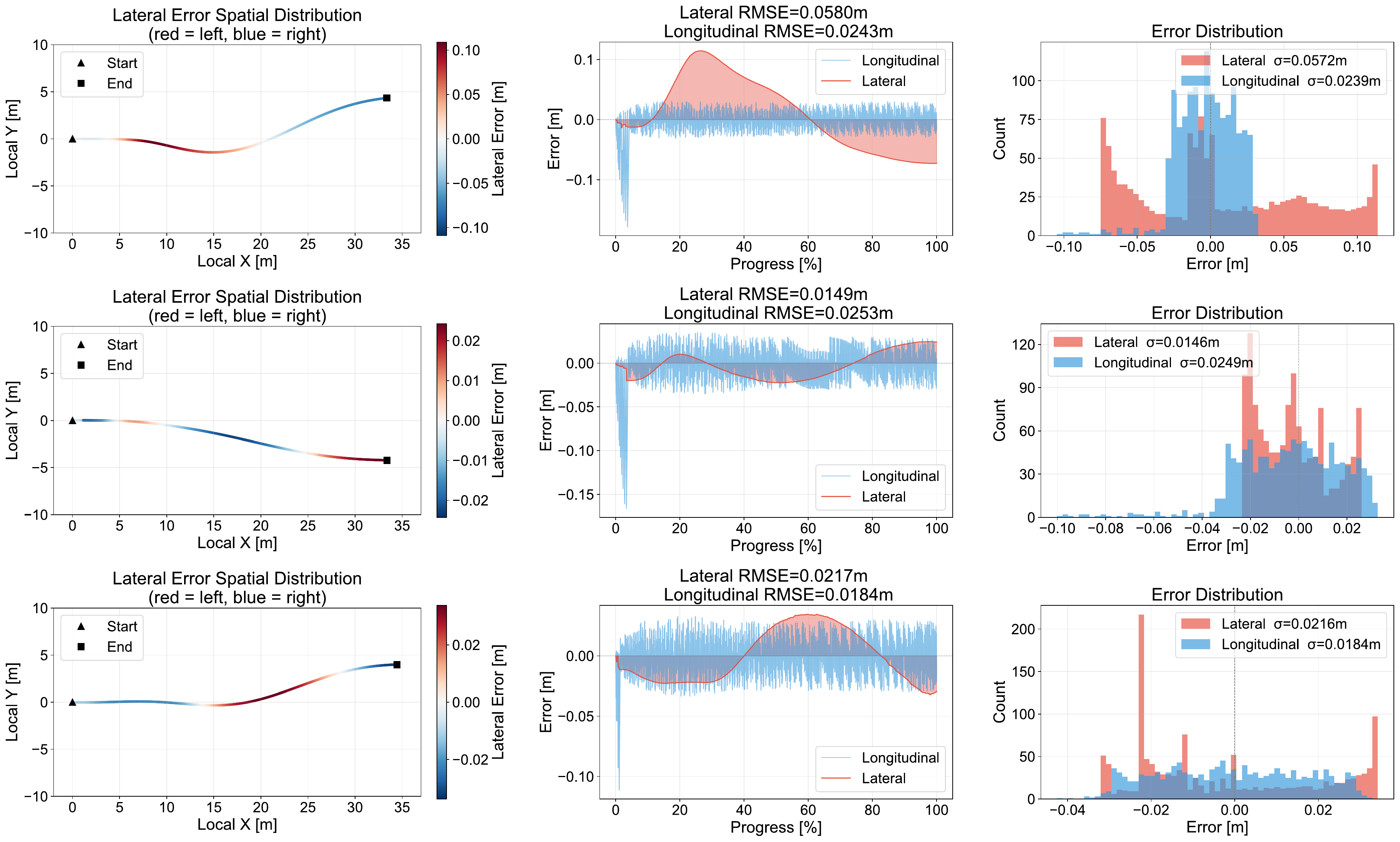}
  \caption{Tracking error analysis for three representative real-world runs. Each row shows the lateral error spatial distribution along the driven path (left), lateral and longitudinal error profiles over path progress (middle), and the corresponding error histograms (right).}\label{fig:track}
\end{figure*}

\begin{table*}[htbp]
  \centering
  \caption{Comparative Analysis of Path Tracking Performance. $\uparrow$ ($\downarrow$) indicates that higher (lower) values are better. \textbf{Bold} values indicate the best results for each metric. }\label{tab:track}
    \begin{tabular*}{0.7\textwidth}{@{\extracolsep{\fill}}lcccc}
      \toprule
      Method & RMSE CTE (\unit{m}) $\downarrow$ & Max CTE (\unit{m}) $\downarrow$ & AHE (\unit{\degree}) $\downarrow$ & $\mathcal{J}_{\text{smooth}}$ $\downarrow$ \\
      \midrule
        IL+Soft       & 0.1073        & 0.2668         & 9.681 & 0.0534 \\
        DC3($t=100$)    & 0.1137        & 0.4850         & 9.237          & 0.0550\\
        ENFORCE($t=50$) & 0.1136        & 0.2766         & 10.15         & 0.0554\\
        DiffSlack (Ours)   & \textbf{0.0802}  & \textbf{0.1685}   & \textbf{8.718}  & \textbf{0.0519} \\ 
      \bottomrule
    \end{tabular*}
\end{table*}

To evaluate whether the planned paths can be reliably executed in practice, we conduct closed-loop tracking experiments in both CARLA simulation and a real-world outdoor environment, as shown in Fig.~\ref{fig:experiment}. The CARLA experiment is used for comparative evaluation across learning-based planners. For each learning-based method, we forward 2,000 collision-free planned paths to the same tracking controller. Implementation details of the CARLA PD controller are provided in Appendix~\ref{app:carla}.

Table~\ref{tab:track} reports the CARLA tracking results. DiffSlack achieves the best performance across all four metrics, showing that its planning advantage translates into better closed-loop executability under the same tracking controller. Although DC3 and ENFORCE improve planning feasibility over IL+Soft in Table~\ref{tab:comparison_static}, their tracking performance is slightly worse. This trend is consistent with their lower curvature or spacing satisfaction in the static evaluation. Their projection corrections can recover obstacle avoidance but may introduce sharper local geometric changes, which makes the resulting paths harder to track. In contrast, DiffSlack preserves feasibility while producing smoother and more uniformly spaced trajectories, leading to the best tracking accuracy and control smoothness.

To further assess deployment feasibility, we validate DiffSlack-planned paths on a physical SUV in an outdoor test site with physical obstacles. We execute eight planned paths under real vehicle dynamics. Fig.~\ref{fig:track} presents detailed error analysis for three representative runs. The lateral errors remain small and are distributed around zero, while the longitudinal errors stay within a narrow range along the path. Across all eight real-world runs, the mean lateral tracking error is 0.046~m and the mean maximum lateral error is 0.152~m, demonstrating that DiffSlack-planned paths remain executable beyond simulation.

\subsection{Sensitivity Analysis and Ablation Study}

\begin{figure*}
  \centering
  \includegraphics[width=0.9\textwidth]{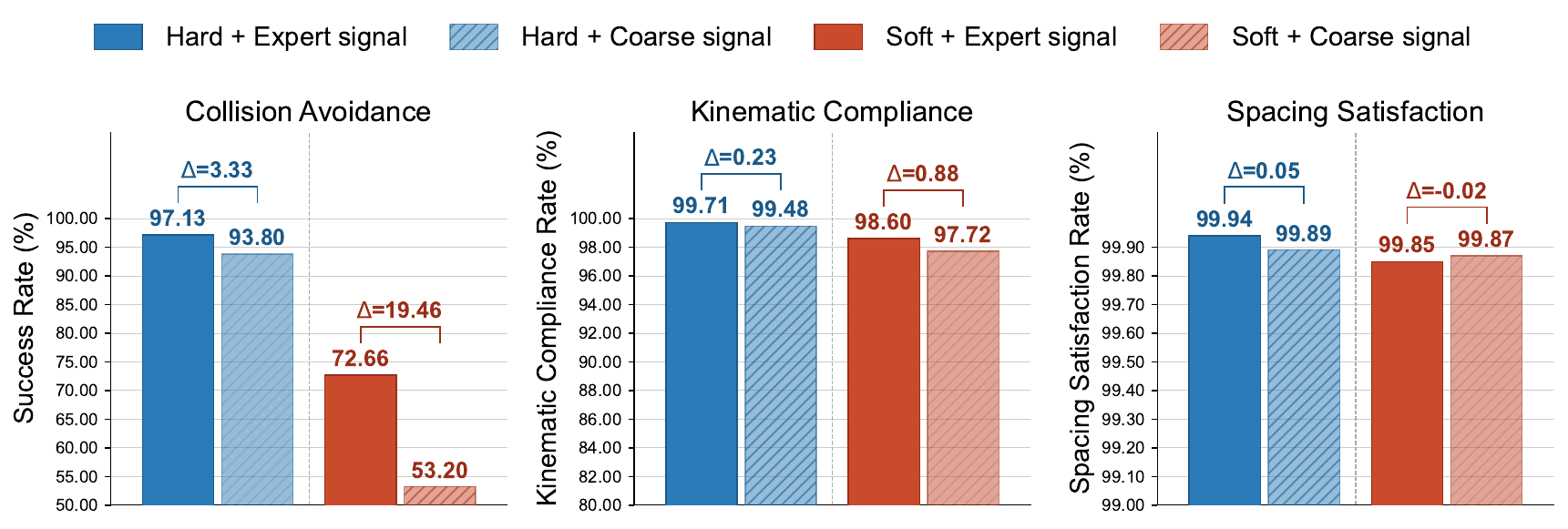}
  \caption{Ablation study on supervision signal quality under a $2 \times 2$ design varying constraint enforcement (Hard vs. Soft) and supervision signal (Expert vs. Coarse). Each group of bars reports the gap $\Delta$ between expert and coarse conditions. Hard constraints compress the supervision quality gap by approximately 83\% on the success rate metric, while kinematic compliance and spacing satisfaction remain largely insensitive to both factors.}\label{fig:ablation_signal}
\end{figure*}

\subsubsection{Supervision Signal Quality}
\label{sec:ablation1}

To examine whether hard projection reduces the dependence on supervision quality, we conduct a \(2\times2\) ablation study by varying the constraint enforcement mechanism (Hard vs. Soft) and the supervision signal quality (Expert demonstrations from NMPC vs. Coarse G-APF guidance). Here, Hard~+~Coarse corresponds to the proposed DiffSlack, and Soft~+~Expert corresponds to IL+Soft in Table~\ref{tab:comparison_static}. The results are shown in Fig.~\ref{fig:ablation_signal}.

The success-rate results show two clear trends. First, hard projection improves collision avoidance under both expert and coarse supervision, indicating that structural constraint enforcement provides a more reliable feasibility mechanism than penalty-based regularization alone. Second, hard projection substantially reduces the performance gap caused by supervision quality. As shown by the annotated gaps in Fig.~\ref{fig:ablation_signal}, hard projection reduces the expert-coarse supervision gap in SR by approximately 83\% compared with soft constraints. This result suggests that the projection layer acts as an implicit corrective mechanism when the external supervision signal is coarse.

The kinematic compliance and spacing satisfaction metrics vary much less across the four settings. These two constraints are already close to saturation for both hard and soft methods, leaving limited room for further improvement. In contrast, collision avoidance involves many obstacle-dependent nonlinear inequalities and is more sensitive to the quality of the learned trajectory. The larger gap reduction on SR therefore indicates that the benefit of hard projection is most pronounced when the constraint set is difficult and supervision is imperfect.

\subsubsection{Necessity of Two-Stage Curriculum Training}

\begin{table}[htbp]
  \centering
  \caption{Ablation study on two-stage curriculum training. SR, APL, and 
  AGD are reported for DiffSlack and DC3 with and without Stage 1 
  pretraining. APL and AGD are computed over collision-free paths 
  only.}\label{tab:ablation_s1}
  \begin{tabular}{@{}llccc@{}}
    \toprule
    Method & Training & SR (\%) $\uparrow$ & APL (m) $\downarrow$ & AGD (m) $\downarrow$ \\
    \midrule
    \multirow{2}{*}{DiffSlack} 
      & w/o Stage 1 & 89.73 & 36.78 & 2.6384 \\
      & w/ Stage 1  & 93.80 & 33.59 & 0.3474 \\
    \midrule
    \multirow{2}{*}{DC3($t=50$)} 
      & w/o Stage 1 & 75.41 & 34.84 & 0.6794 \\
      & w/ Stage 1  & 82.84 & 33.85 & 0.6927 \\
    \bottomrule
  \end{tabular}
\end{table}

\begin{figure*}[htbp]
  \centering
  \includegraphics[width=0.9\textwidth]{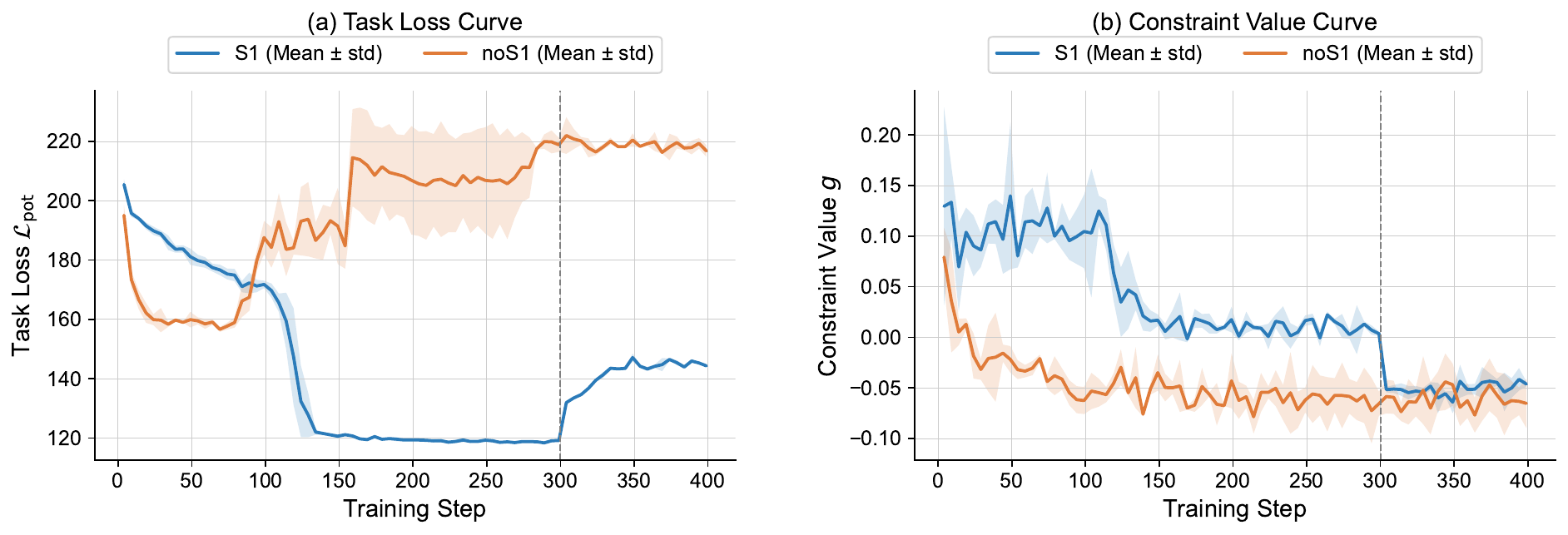}
  \caption{Training dynamics with and without Stage 1 (S1) of the curriculum training strategy, averaged over three random seeds with shaded regions indicating standard deviation. (a) Task loss $\mathcal{L}_\text{pot}$ over training steps. (b) Mean constraint value $g$ over training steps. The vertical dashed line at step 300 marks the transition from Stage 1 to Stage 2 for the S1 variant.}\label{fig:ablation_s1}
\end{figure*}

Table~\ref{tab:ablation_s1} and Figure~\ref{fig:ablation_s1} present the results of removing Stage 1 pretraining from both DiffSlack and DC3($t=50$). In both cases, Stage 1 pretraining consistently improves performance, confirming that soft constraint warm-up is beneficial regardless of the underlying projection mechanism.

The two methods show different degradation patterns when Stage 1 is removed. DiffSlack uses a damped Gauss-Newton projection, which can recover feasibility efficiently but is sensitive to initialization. Without Stage 1, the projected feasible points may correspond to poor task solutions, leading to a moderate drop in SR but a substantial deterioration in APL and AGD. This is consistent with the persistently higher task loss of the no-Stage-1 variant in Fig.~\ref{fig:ablation_s1} (a). DC3 uses gradient-based correction, which changes trajectories more gradually but converges more slowly. Without Stage 1, the global path structure is less severely affected, but the fixed correction budget is more likely to be exhausted before feasibility is recovered, causing a larger drop in SR with smaller changes in APL and AGD. Thus, Stage 1 mainly improves task-relevant projection initialization for DiffSlack and finite-step feasibility recovery for DC3.

\subsubsection{Impact of $I_{\max}$}
\begin{figure*}[htbp]
  \centering
  \includegraphics[width=0.9\textwidth]{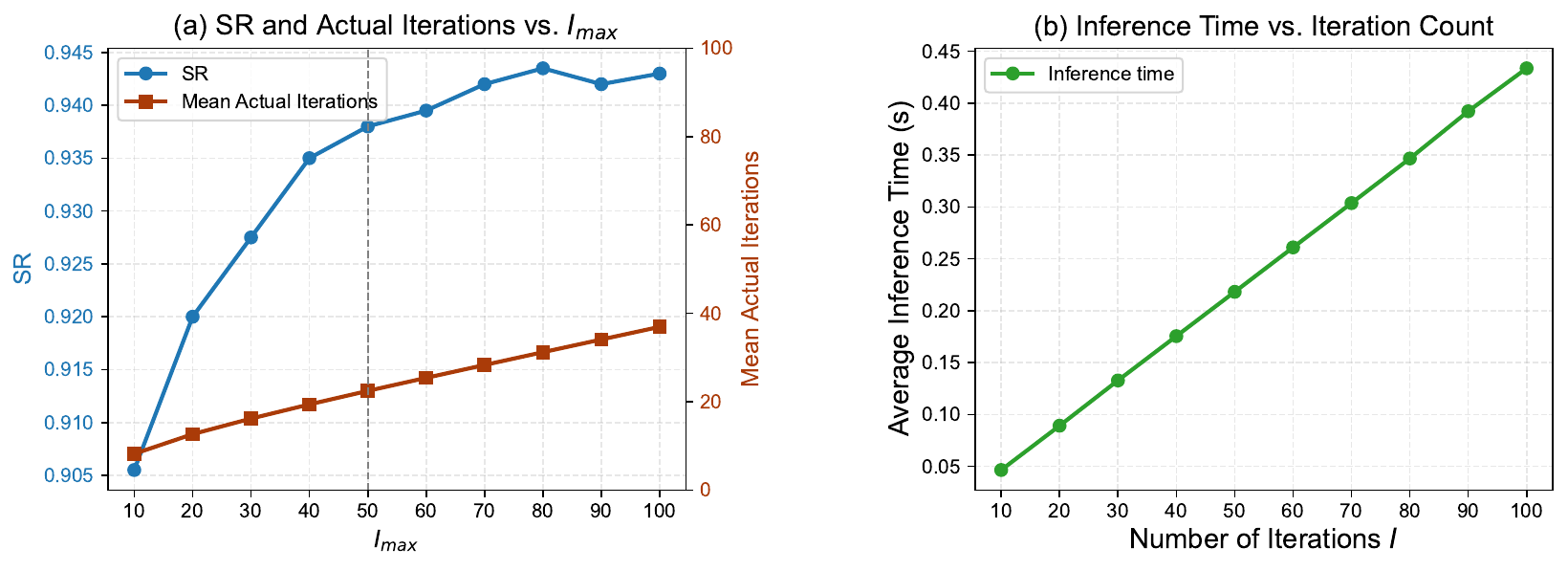}
  \caption{Ablation study on the maximum iteration count $I_{\max}$. (a) SR and mean actual iteration count as a function of $I_{\max}$. The vertical dashed line marks the selected value $I_{\max} = 50$. (b) Average inference time as a function of fixed iteration count $I$, showing a linear relationship.}\label{fig:ablation_imax}
\end{figure*}

Figure~\ref{fig:ablation_imax} examines how the maximum iteration count $I_{\max}$ affects planning performance and computational cost. As shown in Figure~\ref{fig:ablation_imax}(a), SR increases rapidly as $I_{\max}$ grows from 10 to 50, then plateaus with diminishing returns beyond this point. Meanwhile, the mean actual iteration count remains substantially below $I_{\max}$ across all settings, indicating that the early stopping criterion is triggered in most cases before the iteration budget is exhausted. Together, these two observations suggest that $I_{\max} = 50$ strikes a favorable balance between performance and efficiency, and is therefore adopted as the default setting.

Figure~\ref{fig:ablation_imax}(b) shows that average inference time increases linearly with the number of iterations, which is consistent with the fixed per-step computational cost of the Gauss-Newton update. This confirms that $I_{\max}$ serves as a direct and predictable control over inference latency.

\begin{figure*}[htbp]
  \centering
  \includegraphics[width=0.7\textwidth]{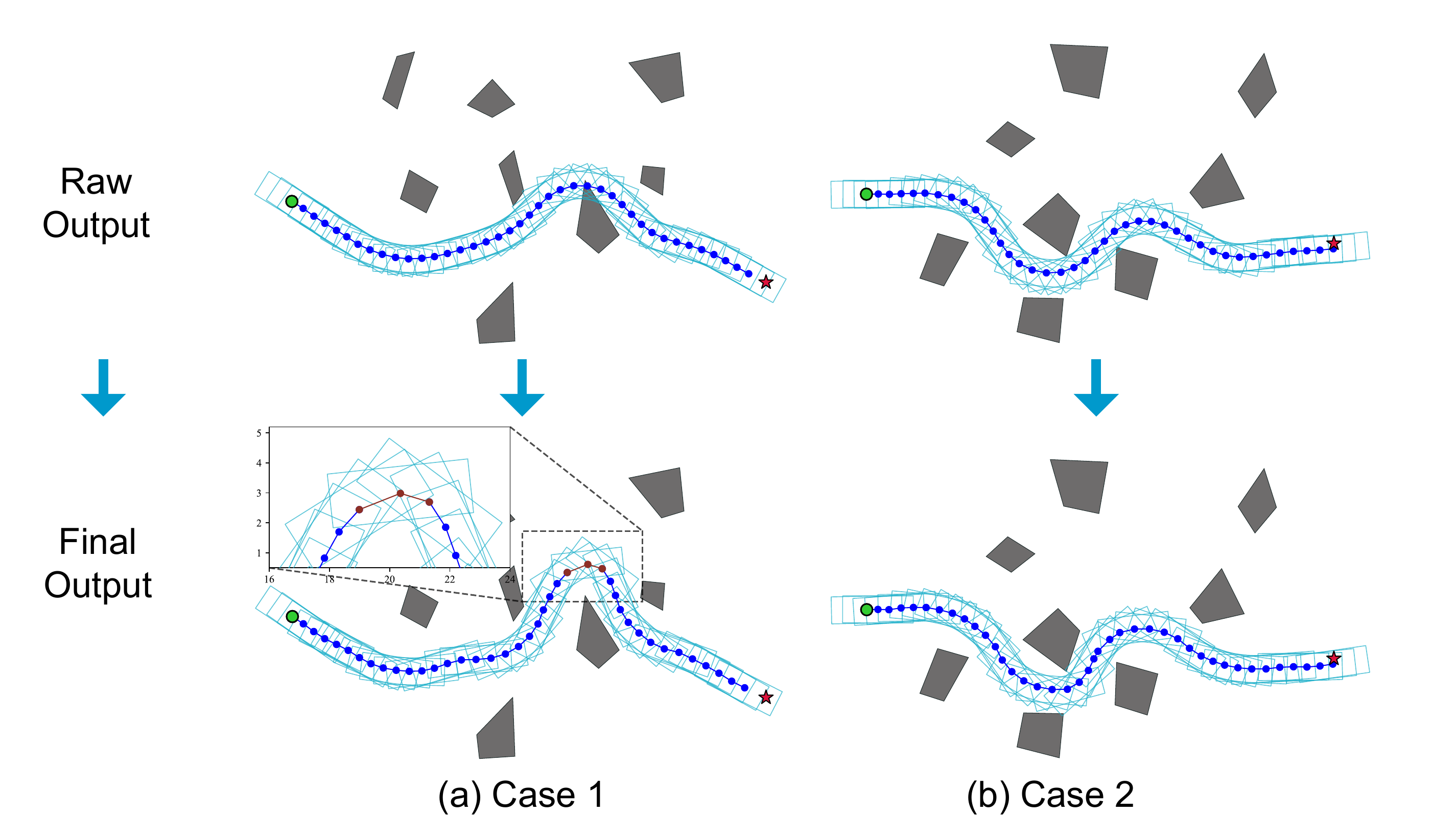}
  \caption{Representative failure cases of DiffSlack. Each column shows the raw output (top) and the final output after Gauss-Newton projection (bottom). (a) Case 1: Large collision correction displaces local waypoints, introducing curvature and spacing violations (red) that cannot be resolved within the iteration budget. (b) Case 2: The raw output selects a route through a dense obstacle cluster. No feasible path exists within this homotopy class and the projection terminates with collision constraints still violated.}\label{fig:failure_cases}
\end{figure*}
\subsection{Failure Case Analysis}

Although DiffSlack achieves a high success rate overall, failure cases still occur under challenging scenarios. Both failure cases in Figure~\ref{fig:failure_cases} stem from network generalization error, though the causes differ.

In Case 1, the raw output follows a correct global route but partially penetrates an obstacle locally. Eliminating the collision violation requires a large lateral displacement of the affected waypoints, which inevitably perturbs the local heading angles and inter-waypoint distances of the surrounding path. This geometric coupling introduces curvature and spacing violations as a byproduct that the remaining iterations cannot fully resolve, leaving the solver at a partially feasible point at $I_{\max}$.

In Case 2, the raw output selects a route passing through a dense obstacle cluster where no feasible path exists. Since the chosen homotopy class contains no feasible solution, local minimum-norm corrections cannot alter the global path topology, and the projection terminates with the collision constraint still violated.

Case 1 is in principle recoverable, since a sufficiently accurate initialization would allow the projection to converge along the correct global route. Case 2 is not recoverable by any local projection method, as no feasible path exists within the homotopy class selected by the network. Addressing Case 1 therefore calls for tighter initialization via stronger Stage 1 training. Addressing Case 2 requires improving the network's global routing decisions, for instance through augmented training coverage of obstacle-dense configurations.

\section{CONCLUSION}

This paper presents DiffSlack, a differentiable projection layer for enforcing general nonlinear inequality constraints in neural networks. By reformulating inequality constraints as equality constraints via learnable slack variables, the network warm-starts the Gauss-Newton projection solver and reduces the correction burden at inference time. A two-stage curriculum training strategy ensures stable optimization and promotes stricter constraint satisfaction. Experiments on a vehicle path planning task with 200 inequality constraints demonstrate that DiffSlack achieves a higher success rate than existing learning-based methods under a comparable computational budget. Furthermore, the hard constraint projection layer acts as an implicit supervision mechanism, significantly reducing sensitivity to supervision signal quality.

Several directions remain for future work. First, relaxing the initialization requirement of the projection solver to handle cases where the raw output lies outside the convergence basin remains an open problem. Second, the proposed framework is general and can potentially be applied to other engineering domains such as robot motion planning and power system optimization.

\appendix
\renewcommand{\theequation}{\thesection.\arabic{equation}}
\renewcommand{\thetable}{\thesection.\arabic{table}}
\renewcommand{\thefigure}{\thesection.\arabic{figure}}

\makeatletter
\@addtoreset{equation}{section}
\@addtoreset{table}{section}
\@addtoreset{figure}{section}
\makeatother
\section{Implementation details of baseline methods}\label{app:baselines}

\subsection{Hybrid A*}
The Hybrid A* algorithm is implemented based on the PythonRobotics library\citep{sakai2018pythonrobotics}.

The Hybrid A* planner searches for kinematically feasible trajectories in the discretized configuration space \((x,y,\theta)\), using motion primitives generated from a discrete steering set. To improve efficiency, Reeds-Shepp analytic expansion is attempted near the goal region or at a fixed iteration frequency. The search cost \(f(n)=g(n)+h(n)\) combines accumulated motion cost with a heuristic obtained from a 2D dynamic-programming distance map. The motion cost penalizes path length, reverse motion, gear changes, steering effort, and steering-rate variation, while polygonal obstacles are inflated for safety and checked for collision along the generated primitives. The hyperparameters are listed in Table~\ref{tab:astar_params}.

\begin{table}[htbp]
    \centering
    \caption{Hyperparameters for the Hybrid A* Planner}
    \label{tab:astar_params}
    \begin{tabular}{lc @{\hskip 1cm} lc}
    \toprule
    \textbf{Parameter} & \textbf{Value} & \textbf{Parameter} & \textbf{Value} \\
    \midrule
    XY Grid Resolution      & 0.5\,m       & Switch Back Cost      & 100.0 \\
    Yaw Grid Resolution     & 15.0$^\circ$ & Backward Cost         & 5.0 \\
    Motion Resolution       & 0.1\,m       & Steer Change Cost     & 5.0 \\
    Steering Discrete Count & 5            & Steer Cost            & 1.0 \\
    Analytic Expansion Dist & 5.0\,m       & Heuristic Cost Weight & 5.0 \\
    Analytic Expansion Freq & 10           &                       & \\
    \bottomrule
    \end{tabular}
\end{table}

\subsection{Informed RRT*}
The Informed RRT* algorithm is also implemented based on the PythonRobotics library\citep{sakai2018pythonrobotics}. The specific hyperparameters used in our implementation are detailed in Table~\ref{tab:rrt_params}.

\begin{table}[htbp]
  \centering
  \caption{Hyperparameters for Informed RRT* Planner}
  \label{tab:rrt_params}
  \begin{tabular}{lc @{\hskip 1cm} lc}
    \toprule
    \textbf{Parameter} & \textbf{Value} & \textbf{Parameter} & \textbf{Value} \\
    \midrule
    Expansion Step Size & 0.5\,m   & Obstacle Inflation           & 1.2\,m \\
    Goal Sampling Rate  & 10\%     & Rewiring Constant ($\gamma$) & 50.0 \\
    Max Iterations      & 1000     & Path Interp. Res.            & 1.0\,m \\
    \bottomrule
  \end{tabular}
\end{table}

\subsection{NMPC}
The NMPC problem is implemented using the CasADi framework and solved via the IPOPT interior-point solver. The optimization horizon is discretized into $T=40$ steps with a sampling time of $\Delta t = \qty{0.5}{s}$. The objective function $J$ is formulated to minimize the deviation from the target state while penalizing control efforts and enforcing trajectory smoothness:
\begin{equation}
    J = w_{g} \|p_T - p_{ref}\|^2 + \sum_{k=0}^{T-1} \left( w_{u} (v_k^2 + \delta_k^2) + w_{s} ((\Delta v_k)^2 + (\Delta \delta_k)^2) \right)
\end{equation}
where $p_T$ is the terminal position, $u_k = [v_k, \delta_k]^T$ are the control inputs (velocity and steering angle), and $\Delta u_k$ represents the rate of change in controls. The optimization is subject to the kinematic bicycle model constraints, actuator limits, and the differentiable collision avoidance constraints described in the main text. To improve convergence, the solver is warm-started with a linear interpolation between the start and goal configurations.

\begin{table}[htbp]
  \centering
  \caption{Hyperparameters for NMPC Planner}
  \label{tab:nmpc_params}
  \begin{tabular}{lc @{\hskip 1cm} lc}
    \toprule
    \textbf{Parameter} & \textbf{Value} & \textbf{Parameter} & \textbf{Value} \\
    \midrule
    Prediction Horizon ($T$) & 40         & Goal Weight ($w_g$)       & 50.0 \\
    Time Step ($dt$)         & 0.5\,s     & Smoothness Weight ($w_s$) & 100.0 \\
    Max Velocity             & 2.0\,m/s   & Input Weight ($w_u$)      & 1.0 \\
    Safety Margin            & 0.25\,m    & Max Acceleration          & 2.0\,m/s$^2$ \\
    \bottomrule
  \end{tabular}
\end{table}

\section{Implementation details of CARLA simulation}\label{app:carla}
\begin{table}[htbp]
    \centering
    \caption{Parameters for CARLA Simulation and Control}
    \label{tab:sim_params}
    \renewcommand{\arraystretch}{1.2}
    \begin{tabular}{lc @{\hskip 1cm} lc}
    \toprule
    \textbf{Parameter} & \textbf{Value} & \textbf{Parameter} & \textbf{Value} \\
    \midrule
    Time Step ($\Delta t$) & 0.02\,s        & Lateral $K_P$ & 3.0 \\
    Max Speed              & 7.2\,km/h     & Lateral $K_D$ & 1.0 \\
    Min Speed              & 3.6\,km/h      & Long. $K_P$   & 0.8 \\
    Max Lat. Accel.        & 2.0\,m/s$^2$   & Long. $K_D$   & 0.0 \\
    \bottomrule
    \end{tabular}
\end{table}
The experiments are conducted in the CARLA simulator using a flat asphalt environment. The simulation operates in synchronous mode with a fixed time step of $\Delta t = \qty{0.02}{s}$ to ensure deterministic physics updates. The ego vehicle is based on the \texttt{vehicle.tesla.model3} blueprint, with the front-wheel maximum steering angle customized to $40^{\circ}$.

Trajectory tracking is executed by a coupled feedforward-feedback controller. Longitudinal control utilizes a PD regulator to track a dynamic reference speed $v_\text{cmd}$ derived from the local path curvature $\kappa$. This target speed is computed as 
\begin{equation}
v_\text{cmd} = \operatorname{clip}\left(\sqrt{a_{\text{lat},\text{max}}/{(|\kappa| + \epsilon)}}, v_{\text{min}}, v_{\text{max}}\right).
\end{equation}
For lateral control, the steering command integrates a PD feedback term with a geometric feedforward term $\delta_{ff} = \arctan(L \cdot \kappa)$, calculated using the kinematic bicycle model where $L$ is the wheelbase.

\section*{Declaration of Generative AI and AI-assisted technologies in the writing process}

During the preparation of this work the authors used Claude in order to refine the language and grammar of the manuscript. After using this tool, the authors reviewed and edited the content as needed and take full responsibility for the content of the publication.

\section*{Acknowledgments}
This work was supported in part by National Natural Science Foundation of China under Grant 52275564, U24A20109, and by Beijing Natural Science Foundation under Grant QY24251.

The authors would like to thank Chengyun Ju, Donggang Sang, and Bingtao Zhang for their valuable discussions and technical support during the course of this work and would also like to thank Jing Ma for the support during the vehicle experiment phase of this work.

\bibliographystyle{cas-model2-names}

\bibliography{references}



\end{document}